\documentclass[11pt]{article}

\usepackage[utf8]{inputenc}
\usepackage[T1]{fontenc}

\usepackage{amsmath,amssymb,amsthm}
\usepackage{bbm}
\usepackage{mathtools}

\usepackage{booktabs}
\usepackage{array}
\usepackage{multirow}
\usepackage{makecell}

\usepackage[margin=1in]{geometry}
\usepackage{enumitem}
\usepackage[none]{hyphenat}   
\usepackage{adjustbox}         
\usepackage[numbers,sort&compress]{natbib}

\usepackage[colorlinks=true,citecolor=blue,linkcolor=blue,urlcolor=blue]{hyperref}

\newcommand{\indicator}{\mathbbm{1}}

\newtheorem{definition}{Definition}
\newtheorem{lemma}{Lemma}
\newtheorem{proposition}{Proposition}
\newtheorem{remark}{Remark}

\title{Why Code, Why Now: An Information-Theoretic Perspective on the Limits of Machine Learning}

\author{Zhimin Zhao\\
Software Analysis and Intelligence Lab (SAIL)\\
School of Computing, Queen's University\\
\texttt{z.zhao@queensu.ca}}
\date{}

\begin{document}
\maketitle

\begin{abstract}
This paper offers a new perspective on the limits of machine learning: the ceiling on progress is set not by model size or algorithm choice but by the information structure of the task itself. Code generation has progressed more reliably than reinforcement learning, largely because code provides dense, local, verifiable feedback at every token, whereas most reinforcement learning problems do not. This difference in feedback quality is not binary but graded. We propose a five-level hierarchy of learnability based on information structure and argue that diagnosing a task's position in this hierarchy is more predictive of scaling outcomes than any property of the model. The hierarchy rests on a formal distinction among three properties of computational problems (expressibility, computability, and learnability). We establish their pairwise relationships, including where implications hold and where they fail, and present a unified template that makes the structural differences explicit. The analysis suggests why supervised learning on code scales predictably while reinforcement learning does not, and why the common assumption that scaling alone will solve remaining ML challenges warrants scrutiny.
\end{abstract}

\section{Introduction}
\label{sec:introduction}

Code is discrete, symbolic, and syntactically unforgiving: a single misplaced character can render an entire program useless, and long-range dependencies span hundreds of lines.
Moreover, correctness standards are absolute: a program either compiles or it does not.
By every intuition about what should be easy for machines, code generation should therefore be among the hardest tasks in artificial intelligence.

Instead, it is where ML has made its most consistent progress.
Large models trained with prediction objectives now write nontrivial programs, refactor codebases, and solve algorithmic problems previously out of reach~\citep{chen2021}.
Meanwhile, reinforcement learning (RL), despite being interactive, closed-loop, and adaptive by design, has produced striking successes in closed domains such as board games~\citep{silver2016} and robotic control.
Recent RL-trained reasoning models~\citep{deepseekr1} have also achieved impressive results on code and mathematics, though invariably by building on supervised pretrained models and coupling RL with structured verification (Section~\ref{sec:expressibility-hurts}).
Yet RL alone consistently struggles to accumulate transferable competence across tasks, even with massive interaction budgets.

The familiar explanations for this disparity do not withstand scrutiny.
``RL needs more compute''~\citep{kaplan2020} does not explain why the pattern persists across orders of magnitude in scale, and ``the reward signal is too sparse''~\citep{sutton2018} does not explain why sparse feedback works for some tasks and not others.
More broadly, the pattern survives improvements in hardware, optimizers, simulators, and representation learning.
This persistence suggests that the primary obstacle is structural rather than architectural.

A widespread misconception has taken hold: that end-to-end neural networks with cutting-edge hardware and Internet-scale data will eventually solve any problem.
The successes in code generation~\citep{chen2021}, Go~\citep{silver2016}, protein structure prediction~\citep{jumper2021}, and image synthesis are taken as evidence that scaling suffices~\citep{kaplan2020}.
But each of these domains shares specific preconditions, objective evaluation functions, bounded solution spaces, massive structured datasets, and dense feedback signals, that align with what gradient descent and pattern matching can exploit.
The domains that remain unsolved, by contrast, resist scaling because their information structure does not support learning.

This paper examines the problems themselves rather than the models, asking what makes a task learnable at scale: not merely computable, expressible, or solvable in principle, but learnable by large models under realistic data and interaction regimes in a way that improves smoothly rather than collapses unpredictably.

We make three contributions:
\begin{enumerate}
    \item We propose a five-level hierarchy of learnability based on information structure, ranging from complete unobservability (Level~0) to deterministic verification (Level~4), and show how it diagnoses when scaling will and will not help.
    \item We distinguish three properties of computational problems (expressibility, computability, and learnability), establish their pairwise relationships, and present a unified formal template that makes the structural differences explicit.
    \item We analyze why supervised learning on code scales predictably while reinforcement learning does not, grounding the explanation in information structure rather than algorithmic or architectural differences.
\end{enumerate}

\section{Related Work}
\label{sec:related}

\paragraph{Language identification and generation in the limit.}
\citet{gold1967} establishes the framework for language identification in the limit, proving that superfinite classes cannot be identified from positive data alone.
\citet{angluin1980} extends this line of work by characterizing identifiable classes through the notion of ``tell-tale'' finite sets.
\citet{kleinberg2024} relax identification to \emph{generation}, asking only whether a learner can eventually produce novel valid strings rather than converge on a grammar.
They show that generation in the limit succeeds for strictly larger classes than identification.
\citet{charikar2024} and, independently, \citet{li2025generation} prove that every countable language collection admits non-uniform generation in the limit; the latter work also introduces the uniform/non-uniform distinction.
\citet{kalavasis2025limits} and \citet{charikar2024} concurrently formalize an inherent validity--breadth tradeoff (hallucination versus mode collapse); \citet{kalavasis2025limits} additionally study a stochastic generation model.
Subsequent work~\citep{compbarriers2025} demonstrates computational feasibility barriers even for simple language classes, connecting generation complexity to classical computability theory.

\paragraph{PAC learning and VC theory.}
Compared with language identification in the limit, a weaker formal model of learnability is Probably Approximately Correct (PAC) learning.
\citet{valiant1984} introduces the PAC framework, which formalizes learnability through sample complexity bounds.
\citet{blumer1989} connect PAC learnability to the Vapnik--Chervonenkis (VC) dimension, establishing that finite VC dimension is both necessary and sufficient for PAC learnability under any distribution.
\citet{shalev2010} unify learnability, stability, and uniform convergence, showing their equivalence for binary classification.

\paragraph{Computability and undecidability.}
\citet{turing1936} establishes the existence of problems no algorithm can solve, with the halting problem as the canonical example.
The relationship between computability and learnability is not one of simple containment: computable functions can be unlearnable (e.g., cryptographic functions), and certain generation tasks succeed on classes that include non-recursive languages~\citep{kleinberg2024}.

\paragraph{Scaling laws and reinforcement learning.}
\citet{chu2025} show that supervised fine-tuning memorizes the training distribution while RL can generalize beyond it, but only with supervised scaffolding.
\citet{zhang2026} demonstrate that stronger supervised checkpoints can underperform weaker ones after RL, due to distribution divergence.
\citet{cui2025} identify policy entropy exhaustion as a fundamental bottleneck in RL for language models.
Independently, \citet{rohatgi2025} propose a computational taxonomy for RL organized by which supervised learning oracles are necessary and sufficient, identifying minimal oracles for Block MDPs and cryptographic hardness barriers for Low-Rank MDPs.
These findings support our thesis that the obstacle is in the problem's information structure, not in model capacity.
At the same time, recent RL-trained reasoning models~\citep{deepseekr1} achieve strong code and mathematics performance, and process reward models~\citep{lightman2023} recover step-level feedback that outcome-based RL discards.
We analyze why these hybrid successes are consistent with, rather than counter to, our information-structure thesis in Section~\ref{sec:expressibility-hurts}.

\paragraph{Structural information and computational observers.}
\citet{jiang2025epiplexity} decompose the information in a dataset into two components relative to a computationally bounded observer: \emph{epiplexity} (structural information extractable into model weights) and \emph{time-bounded entropy} (residual unpredictability). They prove that cryptographically secure pseudorandom generators have nearly maximal time-bounded entropy but negligible epiplexity, formalizing the intuition that computable functions can be unlearnable from observation alone. Empirical measurements show that language data carries far more structural information than image data at comparable compute budgets, and that the same data under different factorizations yields different amounts of learnable structure. These results provide formal backing for the information-structure perspective developed in the present paper.

\paragraph{Goodhart's law and reward misspecification.}
\citet{manheim2019} categorize variants of Goodhart's law.
Recent work~\citep{goodhart2024} formalizes distributional Goodhart effects, and geometric analysis of MDPs~\citep{goodhartmdp2024} shows that Goodhart failure is the default outcome of unconstrained proxy optimization.

\section{What Makes Code Special}
\label{sec:code}

Code exposes information to learning algorithms in a way that few other domains match.
The properties below are not unique to code: natural language has grammar, images have spatial structure, and mathematics has logical constraints. However, code is distinguished by the \emph{degree} to which all three properties co-occur and the \emph{strength} of the verification infrastructure that enforces them. These properties create a rich, dense, and structurally aligned feedback signal that supervised learning can exploit, while the absence of any one of them creates a gap that RL struggles to bridge.

\subsection{Hard Syntactic Constraints}

For conventional statically parsed languages, a deterministic procedure decides in polynomial time whether any string is syntactically valid, yielding a binary answer: valid or invalid~\citep{aho2006}.
A natural language sentence with a dangling modifier still communicates, whereas a program with a mismatched brace or an undeclared variable does not compile in statically checked languages, and even in dynamically typed languages, syntactic violations produce immediate, unambiguous errors.
For a learning system, this distinction matters: every training example of valid code carries a precise signal about the rules of the language, and every violation stands out sharply.

\subsection{Locally Identifiable Errors}

Code is not only globally checkable but also locally structured in ways that correspond to verifiable constraints.
Type consistency, scope rules, interface matching, and dimensional invariants catch errors without examining the entire program.
For instance, a type error points to a particular line and a particular mismatch, and a missing variable declaration points to the exact scope.
These localized diagnostics give learning systems dense, targeted error signals.
By contrast, errors in natural language meaning are diffuse, because no analogous mechanism pinpoints where a sentence goes wrong.

\subsection{Strong Compositionality}

The meaning of a program built from two components depends primarily on the meanings of those components, not on the global context.
A pure sorting function produces the same output regardless of whether it appears inside a web server or a data pipeline, provided it has no side effects or shared state.
Common idioms, function signatures, and algorithmic templates recur across projects.
As a result, patterns learned from one codebase remain useful in another.

\subsection{Why Supervised Learning Outperforms Reinforcement Learning on Code}

Together, these three properties mean that each training example is rich in localized, verifiable, reusable signals.
When a model trains on millions of existing programs, each program serves as a positive example of the language, a collection of reusable local structures, and an implicit vote on the constraints that govern valid code.
The resulting signal is high-dimensional, dense, and structurally aligned with the task.

These same properties also suppress \emph{shortcut learning}: the tendency of gradient descent to latch onto spurious correlations that hold in training data but fail under distribution shift~\citep{geirhos2020}.
In natural-image classification, a model can achieve high accuracy by relying on background texture rather than object shape, because no hard constraint forces it to attend to the causal feature.
Code is different: a spurious surface correlation cannot substitute for syntactic and type correctness.
The compiler enforces causal features, correct syntax, matching types, proper scoping, and rejects outputs that rely on surface statistics alone.
This makes the gap between ``features that correlate with validity'' and ``features that cause validity'' unusually narrow in code, which is precisely the condition a learner needs.

A common intuition holds that code generation should be ideal for RL, since code has a natural reward signal (pass or fail).
In practice, the binary pass/fail reward is low-dimensional (a single bit), easily gamed through hard-coding, and structurally uninformative: it says the program failed but not where or why.
In its default form, the binary reward acts as a filter that accepts or rejects completed outputs rather than as a teacher that guides the model toward correct intermediate steps.
Shaped rewards and auxiliary signals can narrow this gap, but they require domain-specific engineering that reintroduces the very structure supervised learning gets for free.
In supervised training, the model sees the correct answer at every step, whereas in RL it sees only whether the completed program passes.
The resulting difference in information density between a single pass/fail bit and full token-level supervision remains substantial.

\subsection{Formal Grounding}

These properties align with a formal framework developed by \citet{gold1967} and extended by \citet{kleinberg2024}.
A learner that observes only valid strings from an unknown language can eventually generate new valid strings, even when it cannot identify which language it is observing.
We formalize this distinction in Section~\ref{sec:hierarchy} and show how it grounds code generation's tractability.

\section{A Hierarchy of Learnability}
\label{sec:hierarchy}

When a learning system fails, we typically ask: is the model too small? The optimizer too weak? The data too scarce?
These questions assume the problem is solvable and that the bottleneck is engineering.
But learning problems can fail for entirely different reasons, and conflating them leads to wasted resources and misplaced optimism.

The critical concept is \emph{information structure}, defined as the mechanism by which the world exposes distinguishing information to a learner.
It determines whether that information exists at all, when it appears, and whether it can be verified.
Information structure constrains learning more fundamentally than model size, algorithm choice, or computational budget, because no amount of any of the latter can overcome the absence of the former.
\citet{jiang2025epiplexity} formalize a related intuition by decomposing dataset information into \emph{structural information} (learnable patterns extractable into model weights) and \emph{time-bounded entropy} (residual unpredictability) relative to a computationally bounded observer. Two datasets with identical total information can differ in how much structure a bounded learner can extract, and it is the structural component that determines what is learnable.

With this concept, we distinguish five levels of learnability, ordered by the quality of feedback available to the learner, from no signal at all to immediate deterministic verification (Table~\ref{tab:hierarchy}).
These levels are not sharp boundaries. Just as phase transitions in combinatorial problems (such as the SAT solvability threshold near clause density $\alpha \approx 4.26$~\citep{xu2000}) turn a binary classification into a probabilistic landscape, a real-world task often sits in a gradient between adjacent levels, and small changes in problem structure can shift it from one regime to another.

\begin{table}[ht]
\centering
\caption{Hierarchy of feedback quality across five learning settings, from fully unobservable to deterministically verifiable.}
\label{tab:hierarchy}
\adjustbox{max width=\textwidth}{%
\begin{tabular}{@{}cp{2.2cm}p{6.5cm}p{2.5cm}p{4cm}@{}}
Level & Feedback quality & What happens & Scaling outcome & Example \\
\midrule
0 & None & Unobservability: different hypotheses produce identical observations. No data helps. & Impossible & The halting problem, fully Goodharted metrics \\[0.8em]
\midrule
1 & Adversarial & Adversariality and reflexivity: distinguishing information exists but the target shifts in response to learning. & Unstable & Gaming a ranking algorithm, adversarial online learning \\[0.8em]
\midrule
2 & Noisy & Stochasticity: hypotheses are statistically distinguishable, but each observation is noisy. & Data-dependent & Image classification, spam detection \\[0.8em]
\midrule
3 & Indirect & One-sided evidence: wrong hypotheses are eventually falsified, but correctness is never confirmed. & Convergent but unconfirmed & Formal language learning, program testing \\[0.8em]
\midrule
4 & Direct & Every output can be immediately and deterministically verified. & Predictable & Type checking, compilation, formal proof verification \\
\end{tabular}}
\end{table}

The hierarchy rests on three formal properties of computational problems, expressibility, computability, and learnability, that characterize different aspects of what machines can and cannot do, and confusing them leads to misdiagnoses of difficulty.
We define the first two properties as building blocks, then integrate the formal definitions of learnability into the hierarchy levels where they apply.

\subsection{Formal Foundations}
\label{sec:three-properties}

\subsubsection{Standing Assumptions}
\label{sec:standing}

All definitions share the following setup.

\begin{itemize}
    \item \textbf{Instance space.}
    Let $X$ be a countable set.
    When $X = \Sigma^*$ for a finite alphabet~$\Sigma$, computability is defined in the standard Turing sense.
    When $X$ is a countable subset of $\mathbb{R}^d$, all computational definitions are understood relative to a fixed computable encoding $\mathrm{enc}\colon X \to \{0,1\}^*$. Computability over richer domains (e.g., all of $\mathbb{R}^d$) requires an effective representation in the sense of Type-2 computability~\citep{weihrauch2000}, which we do not pursue here.

    \item \textbf{Target language.}
    A \emph{language} is a set $L \subseteq X$.
    Its indicator function is $\chi_L\colon X \to \{0,1\}$, with $\chi_L(x) = 1$ iff $x \in L$.

    \item \textbf{Function classes.}
    A \emph{function class} $\mathcal{F} \subseteq \{f\colon X \to \{0,1\}\}$ is a collection of total, binary-valued functions on $X$.
    Since outputs lie in $\{0,1\}$, all functions are bounded. Measurability with respect to any distribution on $X$ holds automatically for the discrete $\sigma$-algebra.

    \item \textbf{Concept class.}
    A \emph{concept class} $\mathcal{C}$ is a collection of languages.
    We write $L \in \mathcal{C}$ to denote membership.
\end{itemize}

\subsubsection{Expressibility}

\begin{definition}[Expressibility]
\label{def:expressibility}
A language $L$ is \emph{expressible} within a function class $\mathcal{F} \subseteq \{f\colon X \to \{0,1\}\}$ if
\begin{equation}
\label{eq:expr}
    \exists\, f \in \mathcal{F} \quad \forall\, x \in X \quad f(x) = \chi_L(x).
\end{equation}
Equivalently, the risk functional
\begin{equation}
\label{eq:risk-expr}
    R_{\mathrm{expr}}(f, L) \;=\; \sup_{x \in X}\, \bigl|f(x) - \chi_L(x)\bigr|
\end{equation}
satisfies $R_{\mathrm{expr}}(f, L) = 0$ for some $f \in \mathcal{F}$.
\end{definition}

Expressibility asks only whether a correct classifier exists in $\mathcal{F}$.
In particular, it imposes no requirement that the classifier be computable or discoverable from data.
The quantifier structure is minimal: $\exists f\;\forall x$.

\begin{remark}[Expressibility is relative]
\label{rem:expr-relative}
Expressibility is always relative to a restricted function class~$\mathcal{F}$.
The unrestricted class $\{f\colon X \to \{0,1\}\}$ trivially contains a correct classifier for every language, including undecidable ones. Without restriction on $\mathcal{F}$, expressibility is vacuous.
The notion becomes informative only when $\mathcal{F}$ is constrained: linear classifiers express half-spaces, deterministic finite automata express regular languages, and neural networks of bounded depth express specific decision regions.
\end{remark}

\begin{remark}[Well-definedness of the supremum]
\label{rem:sup-welldef}
Since $f$ and $\chi_L$ both map into $\{0,1\}$, the difference $|f(x) - \chi_L(x)|$ takes values in $\{0,1\}$, and the supremum in \eqref{eq:risk-expr} is well-defined without additional measurability or boundedness conditions.
\end{remark}

\begin{remark}[Risk as zero--one uniform loss]
\label{rem:risk-zeroone}
For binary-valued functions, $R_{\mathrm{expr}}(f,L) \in \{0,1\}$ and the condition $R_{\mathrm{expr}} = 0$ reduces to pointwise equality $f = \chi_L$.
The risk formulation adds no analytical power beyond existential equality in the binary case. We retain it because it instantiates the unified template of Section~\ref{sec:template}, where the same schema accommodates distributional and sequential risk functionals that do not reduce to exact equality.
\end{remark}

\subsubsection{Computability}

\begin{definition}[Computability]
\label{def:computability}
A language $L$ is \emph{computable} (decidable) if there exists a \emph{total} Turing machine $M$ (one that halts on every input) such that
\begin{equation}
\label{eq:comp}
    \forall\, x \in X, \quad M\bigl(\mathrm{enc}(x)\bigr) = \chi_L(x).
\end{equation}
Equivalently, within the class $\mathcal{M}_{\mathrm{total}}$ of total Turing machines, the risk functional
\begin{equation}
\label{eq:risk-comp}
    R_{\mathrm{comp}}(M, L) \;=\; \sup_{x \in X}\, \bigl|M\bigl(\mathrm{enc}(x)\bigr) - \chi_L(x)\bigr|
\end{equation}
satisfies $R_{\mathrm{comp}}(M, L) = 0$ for some $M \in \mathcal{M}_{\mathrm{total}}$.
\end{definition}

\begin{remark}[Totality is essential]
\label{rem:totality}
The risk $R_{\mathrm{comp}}$ is well-defined only for total machines.
If $M$ does not halt on some input $x_0$, then $M(\mathrm{enc}(x_0))$ is undefined and $R_{\mathrm{comp}}$ cannot be evaluated.
The equivalence ``$L$ is computable iff $R_{\mathrm{comp}}(M,L) = 0$ for some $M$'' holds precisely within $\mathcal{M}_{\mathrm{total}}$.
\end{remark}

Computability asks whether a correct, terminating algorithm exists.
The quantifier structure matches expressibility ($\exists M\;\forall x$) but restricts the mechanism class from arbitrary mathematical functions to effective procedures.
Expressibility quantifies over functions that may exist only as abstract mathematical objects, whereas computability quantifies over procedures that can be executed step by step.

A weaker notion, \emph{recursive enumerability} (r.e.\ or semi-decidability), relaxes totality: $M$ must halt and accept on inputs in $L$ but may run forever on inputs outside $L$.
This asymmetry reappears in Level~3, where wrong hypotheses are eventually falsified but correctness is never confirmed.

The computable languages are strictly contained in the expressible ones ($\textup{Computable} \subsetneq \textup{Expressible}$).
Every total Turing machine is a member of $\{f\colon X \to \{0,1\}\}$, so every computable language is expressible. The halting language $H = \{\langle P \rangle : P \text{ halts on empty input}\}$ witnesses strict containment: $\chi_H$ exists as a mathematical function, but no total Turing machine computes it~\citep{turing1936}.

\medskip

With these formal building blocks in place, we now describe each level of the hierarchy. Levels~0 and~1 are characterized by information-theoretic and strategic obstacles that prevent convergence. Levels~2 and~3 correspond to PAC learnability and generation in the limit, respectively. Level~4 is characterized by deterministic verification.

\subsection{Level 0: No Feedback}

No interaction protocol of any kind can distinguish competing hypotheses.
The obstacle is \emph{information-theoretic indistinguishability}: different hypotheses produce identical observations under every possible interaction, regardless of the protocol.
The halting problem is a canonical case, but unobservability also arises when every available proxy is many-to-one with respect to the true objective.
In a completely Goodharted system~\citep{manheim2019}, every metric the learner observes is compatible with both genuine improvement and strategic gaming.
Recent formalization~\citep{goodhart2024} shows that the tail distribution of proxy--objective discrepancies determines whether proxy over-optimization becomes merely useless or actively harmful. Heavy-tailed misalignment makes strong Goodhart degradation not just possible but probable.

\subsection{Level 1: Adversarial Feedback}

Separating information exists and a sufficiently adaptive protocol could in principle extract it, but the environment actively prevents convergence.
The obstacles are \emph{adversariality} and \emph{reflexivity}: the environment works against the learner, or the target shifts in response to learning.
Specifically, an adversarial environment can delay counterexamples or adapt to the learner's strategy, while reflexivity creates a related pathology in which the learner chases an object that moves because it is being chased.
Platform ecosystems where participants game the ranking algorithm are a reflexive example~\citep{perdomo2020}: the ``correct'' ranking changes precisely because the algorithm is deployed.
Formally, a Level~1 task is one where the data-generating distribution at step $t+1$ depends on the learner's current hypothesis: $D_{t+1} = \mathcal{E}(h_t)$, where $\mathcal{E}$ is an environment response function.
This non-stationarity means convergence guarantees from the PAC framework (which assumes a fixed $D$) do not apply, even when the hypothesis class has finite VC dimension.
PAC convergence relies on the uniform law of large numbers over i.i.d.\ samples~\citep{blumer1989}. When $D_{t+1} = \mathcal{E}(h_t)$, consecutive samples come from different distributions, so the empirical risk $\hat{R}_m(h)$ no longer estimates $R(h, L, D)$ for any single $D$.
Learning under non-stationarity requires additional structural assumptions (e.g., bounded drift or mixing conditions~\citep{besbes2015}) that are absent in the reflexive setting.

\subsection{Level 2: Noisy Feedback}

The obstacle is \emph{stochasticity}: different hypotheses are statistically distinguishable, and the learner can converge in a probabilistic sense, but individual observations are noisy.
The PAC framework~\citep{valiant1984} formalizes this level.

\begin{definition}[PAC learnability]
\label{def:pac}
A concept class $\mathcal{C}$ over $X$ is \emph{PAC-learnable} if there exists a learning algorithm $A$ and a hypothesis class $\mathcal{H}$ of measurable functions $X \to \{0,1\}$ (with $\mathcal{H} \supseteq \mathcal{C}$, see Remark~\ref{rem:realizable}) such that for every $L \in \mathcal{C}$, every probability distribution $D$ over $X$, and every $\varepsilon, \delta > 0$, there exists $m = \mathrm{poly}(1/\varepsilon, 1/\delta)$ such that, given $m$ i.i.d.\ samples $S \sim D^m$, the algorithm outputs a hypothesis $h = A(S) \in \mathcal{H}$ satisfying
\begin{equation}
\label{eq:pac}
    \Pr_{S \sim D^m}\!\Bigl[\,R_{\mathrm{PAC}}(h, L, D) \leq \varepsilon\,\Bigr] \;\geq\; 1 - \delta,
\end{equation}
where the risk functional is
\begin{equation}
\label{eq:risk-pac}
    R_{\mathrm{PAC}}(h, L, D) \;=\; \Pr_{x \sim D}\bigl[h(x) \neq \chi_L(x)\bigr].
\end{equation}
\end{definition}

\begin{remark}[Realizable case]
\label{rem:realizable}
Definition~\ref{def:pac} operates under the \emph{realizable} assumption: the hypothesis class contains a perfect classifier ($\mathcal{H} \supseteq \mathcal{C}$), so zero error is achievable in principle.
The \emph{agnostic} setting drops this assumption and minimizes over $\mathcal{H}$ without guaranteeing zero risk. Agnostic learnability is strictly harder and requires different sample complexity bounds~\citep{shalev2014}.
We adopt the realizable case throughout because our focus is on structural barriers to learnability, which are present even under the most favorable assumptions.
\end{remark}

\begin{remark}[Improper learning]
\label{rem:improper}
Definition~\ref{def:pac} permits \emph{improper} learning: the output hypothesis $h$ may lie in $\mathcal{H} \supseteq \mathcal{C}$ rather than in $\mathcal{C}$ itself.
Proper learning (requiring $h \in \mathcal{C}$) is strictly stronger and computationally harder for some concept classes~\citep{shalev2014}.
\end{remark}

\begin{remark}[Sample complexity and VC dimension]
\label{rem:vc-sample}
The polynomial bound $m = \mathrm{poly}(1/\varepsilon, 1/\delta)$ in Definition~\ref{def:pac} hides a critical structural parameter.
The fundamental theorem of PAC learning~\citep{blumer1989,shalev2010} establishes that a concept class is PAC-learnable if and only if its VC dimension $d = \mathrm{VC}(\mathcal{H})$ is finite, in which case the sample complexity is
\[
    m = O\!\left(\frac{d + \log(1/\delta)}{\varepsilon}\right).
\]
When $d = \infty$, no finite sample suffices for uniform convergence, and PAC learnability fails.
The VC dimension thus mediates the connection between expressibility and learnability formalized in Section~\ref{sec:expressibility-hurts}.
\end{remark}

\begin{remark}[Measurability]
\label{rem:measurability}
The probability in \eqref{eq:risk-pac} requires that $\{x : h(x) \neq \chi_L(x)\}$ be $D$-measurable.
Since $h$ and $\chi_L$ are $\{0,1\}$-valued, measurability holds whenever $\mathcal{H}$ consists of measurable functions with respect to the $\sigma$-algebra on which $D$ is defined, which we assume throughout.
\end{remark}

\begin{remark}[PAC learnability implies computability of evaluation]
\label{rem:pac-comp}
PAC learnability requires that $A$ be a total computable procedure (one that halts on every input) and that each hypothesis $h = A(S)$ be evaluable on new inputs.
Therefore, PAC learnability entails that hypotheses output by the learner are computable procedures. Each such $h$ approximates $L$ to within $\varepsilon$ error under the distribution $D$, but need not decide $L$ exactly.
The converse fails: computable languages need not be PAC-learnable (Section~\ref{sec:pairwise}).
\end{remark}

Distinguishing a coin that lands heads 49\% of the time from one at 51\% requires thousands of flips, but it is always possible, and most conventional machine learning operates at this level.
Stochasticity raises the price of learning but does not destroy learnability.

\subsection{Level 3: Indirect Feedback}

The obstacle is \emph{one-sidedness}: the learner receives real information, but only from one side.
Positive evidence accumulates, but explicit correction never arrives.
A wrong hypothesis is not immediately flagged, but it cannot survive indefinitely: when contradicting evidence appears, it provides unambiguous falsification.
The central asymmetry is that correctness is never confirmed, so the learner improves monotonically but can never know it has converged.

The classical theory of language identification asks the stronger question: can the learner eventually identify which language it is observing?
\citet{gold1967} shows this is impossible for most language classes from positive examples alone.
Without negative examples, the learner cannot distinguish between two languages where one is a subset of the other.

However, \citet{kleinberg2024} reframe the question.
Instead of identification, they ask whether the learner can generate valid new strings indefinitely.
Generation succeeds for strictly larger classes of languages than identification does, because generation requires only that the learner stay within some infinite safe subset of the target language, whereas identification requires characterizing the language's complete boundary.
In practice, a code generation model does not need to recover the complete grammar of Python but only to keep producing valid programs, a task with strictly weaker requirements.

\begin{definition}[Generation in the limit]
\label{def:gen}
Let $\mathcal{C}$ be a concept class of \emph{infinite} languages.
A generator $G\colon X^{<\infty} \to X$, where $X^{<\infty} = \bigcup_{n=1}^{\infty} X^n$, \emph{succeeds} on $\mathcal{C}$ if, for every $L \in \mathcal{C}$ and every enumeration $\sigma = (x_1, x_2, \ldots)$ satisfying $\{x_n : n \in \mathbb{N}\} = L$, there exists $N \in \mathbb{N}$ such that for all $n \geq N$:
\begin{align}
\label{eq:gen-valid}
    G(x_1, \ldots, x_n) &\in L \qquad\text{(validity),} \\
\label{eq:gen-novel}
    G(x_1, \ldots, x_n) &\notin \{x_1, \ldots, x_n\} \qquad\text{(novelty).}
\end{align}
The risk functional captures \emph{validity} only:
\begin{equation}
\label{eq:risk-gen}
    R_{\mathrm{gen}}(G, L, \sigma) \;=\; \limsup_{n \to \infty}\; \indicator\bigl\{G(\sigma_{\leq n}) \notin L\bigr\}.
\end{equation}
Success requires both $R_{\mathrm{gen}} = 0$ (eventual perpetual validity) and eventual perpetual novelty.
\end{definition}

\begin{remark}[Infinite language requirement]
\label{rem:infinite}
The novelty condition~\eqref{eq:gen-novel} requires that $L$ be infinite.
If $L$ is finite, then after all elements of $L$ have appeared in the enumeration, no novel valid output exists.
We restrict $\mathcal{C}$ to infinite languages throughout.
\end{remark}

\begin{remark}[Validity versus productivity]
\label{rem:validity-productivity}
$R_{\mathrm{gen}}$ captures whether the generator eventually stays within $L$ (semantic convergence).
Novelty is a separate structural constraint enforcing nontrivial generativity (productivity).
We separate the two because they fail for different reasons: a generator may produce only valid outputs but repeat itself (validity without productivity), or produce novel outputs that eventually leave $L$ (productivity without validity).
For completeness, one can define a companion novelty risk:
\begin{equation}
\label{eq:risk-nov}
    R_{\mathrm{nov}}(G, L, \sigma) \;=\; \limsup_{n \to \infty}\; \indicator\bigl\{G(\sigma_{\leq n}) \in \{x_1, \ldots, x_n\}\bigr\}.
\end{equation}
Full success then requires both $R_{\mathrm{gen}} = 0$ and $R_{\mathrm{nov}} = 0$.
We retain $R_{\mathrm{gen}}$ alone in the unified template (Table~\ref{tab:unified}) because validity and novelty are structurally independent constraints with distinct failure modes, and the template is designed to capture the risk of incorrect outputs rather than the productivity of the mechanism.
\end{remark}

\begin{remark}[Strength of universal quantification over enumerations]
\label{rem:adversarial-enum}
Definition~\ref{def:gen} quantifies over \emph{every} enumeration of $L$, including adversarially ordered ones.
The generator receives no structural guarantee about the order in which elements arrive.
This corresponds to Gold-style learning under arbitrary presentation~\citep{gold1967}, adapted to the generation setting by \citet{kleinberg2024}.
The requirement is correspondingly strong: the generator must succeed regardless of how the environment sequences its data.
\end{remark}

\begin{remark}[Sequence semantics]
\label{rem:sequence}
The index $n$ in Definition~\ref{def:gen} is a sequence position in $\mathbb{N}$, not a physical time step.
The generator operates on finite prefixes $\sigma_{\leq n} = (x_1, \ldots, x_n)$ of infinite sequences.
``In the limit'' is a statement about convergence over sequence prefixes, not about the passage of physical time.
\end{remark}

\begin{definition}[Identification in the limit]
\label{def:ident}
Let $\{L_i\}_{i \in \mathbb{N}}$ be an \emph{effective} indexing of $\mathcal{C}$: the membership predicate $x \in L_i$ is uniformly decidable given $i$ and $x$ (i.e., there exists a total Turing machine that on input $(i, x)$ outputs $\chi_{L_i}(x)$).
An identifier $A$ \emph{succeeds} on $\mathcal{C}$ if, for every $L \in \mathcal{C}$ and every enumeration $\sigma$ of $L$, there exists $N$ such that for all $n \geq N$,
\begin{equation}
    A(x_1, \ldots, x_n) = i \quad \text{where} \quad L_i = L.
\end{equation}
The identifier must eventually lock onto the correct language index and never change.
\end{definition}

Identification is strictly stronger than generation: every concept class identifiable in the limit is also generable, but the converse fails.
An identifier that has converged to the correct index can be converted into a generator by enumerating novel elements of the identified language.
For the separation, \citet{gold1967} shows that superfinite classes (those containing all finite languages and at least one infinite language) are not identifiable from positive data, yet \citet{kleinberg2024} exhibit superfinite classes that are generable in the limit.

\subsection{Level 4: Direct Feedback}

At this level, no information-theoretic obstacle remains, because every output can be immediately and deterministically judged correct or incorrect.
Type checking, compilation, and formal theorem verification~\citep{demoura2021} operate at this level.
Level~4 exploits a deep asymmetry: verification is cheap while search is hard.
A type checker confirms correctness in linear time; finding a type-correct program may require exploring an exponential space.
This is the P versus NP gap working in the learner's favor: the learner inherits the easy side of the asymmetry.
Level~4 does not imply greater intelligence on the learner's part but rather a stronger verification structure in the task.

\subsection{Code Generation Across Levels}

Code generation illustrates how a single task can draw on multiple levels.
During training, the model learns from a stream of valid programs without explicit negative examples, a Level~3 information structure.
But code's verification infrastructure (compilers, type checkers, test suites) provides Level~4 feedback that leaks into the learning process: training data is filtered by compilation, reward signals come from test results, and error diagnostics point to specific failures.
Unlike most ML domains where training data is passively observed, code permits \emph{active experimentation}: the learner (or its training pipeline) can execute a candidate program and observe whether it passes, producing feedback that is causally tied to the specific output rather than merely correlated with it.
Most ML training data provides only observational signal ($P(Y \mid X)$); code's execution infrastructure provides a form of experimental signal in which the outcome is deterministically linked to the input, which is why its feedback is so much more informative.
This combination (a Level~3 learning problem scaffolded by Level~4 verification with interventional feedback) is what makes code generation tractable.
RL on code can access the same Level~4 verifiers, but it collapses the rich diagnostic signal into a single scalar reward, losing the density and locality that make supervised learning effective.
Moreover, most learning tasks have no such scaffold.

Code generation models do not ``understand'' their programming languages in any complete sense.
Indeed, they cannot reliably decide whether an arbitrary program is valid.
Yet they stably generate valid programs, drawing on patterns absorbed from training data.
This behavior reflects not a failure of understanding but a different, achievable kind of competence, determined by the information structure of the task.

The purpose of the hierarchy is diagnostic: when a learning problem fails, the first question should not be ``do we need a bigger model?'' but ``what kind of feedback does the learner actually receive?''
For example, a Level~0 problem cannot be solved by scaling, and a Level~1 problem cannot be solved by better optimizers.
Recognizing the level is therefore the prerequisite for choosing the right response.

Crucially, the classification can be performed \emph{before} scaling is attempted.
Four observable properties, checkable from task specifications alone, determine the level:
(i)~whether any interaction protocol can distinguish competing hypotheses (if not, Level~0);
(ii)~whether the data distribution is stationary or shifts in response to the learner (non-stationarity implies Level~1);
(iii)~whether both confirming and disconfirming evidence are available (symmetric feedback indicates Level~2; one-sided positive evidence indicates Level~3);
(iv)~whether every candidate output can be immediately and deterministically verified (if so, Level~4).
For instance, automated theorem proving in Lean~4~\citep{demoura2021} satisfies all four criteria for Level~4 before any model is trained: proof candidates are distinguishable, the proof-checking rules are fixed, the type checker provides both acceptance and rejection, and verification is deterministic.
The hierarchy thus predicts that scaling should yield reliable progress on theorem proving, a prediction consistent with recent empirical results.

\subsection{Quantifier Depth}
\label{sec:quantifiers}

The quantifier structure deepens across definitions, and the depth reflects the adversarial robustness required of the mechanism (Table~\ref{tab:quantifiers}).

\begin{table}[t]
\centering
\caption{Quantifier alternation structure and depth for each property in the learnability hierarchy.}
\label{tab:quantifiers}
\adjustbox{max width=\textwidth}{%
\begin{tabular}{@{}lll@{}}
Property & Quantifier structure & Alternation depth \\
\midrule
Expressibility & $\exists f\;\forall x$ & 2 \\
\midrule
Computability & $\exists M\;\forall x$ \;(+ totality) & 2 (+ constraint) \\
\midrule
PAC learnability & $\exists A\;\forall L\;\forall D\;\forall \varepsilon,\delta\;\exists m$ & 4 \\
\midrule
Generation in the limit & $\exists G\;\forall L\;\forall \sigma\;\exists N\;\forall n \geq N$ & 5 \\
\midrule
Identification in the limit & $\exists A\;\forall L\;\forall \sigma\;\exists N\;\forall n \geq N$ & 5 \\
\end{tabular}}
\end{table}

Each additional quantifier alternation represents an additional degree of adversarial choice the mechanism must handle.
At the shallowest level, expressibility must work for all instances but faces no adversary choosing instances adaptively.
PAC learnability, in turn, must work for all target concepts, all distributions, and all accuracy requirements.
At the deepest level, generation and identification must work for all targets and all orderings of the data, including adversarial ones.

The properties also differ in what counts as an adversary and how success is quantified (Table~\ref{tab:regimes}).

\begin{table}[bp]
\centering
\caption{How success is measured and what adversary each property must withstand.}
\label{tab:regimes}
\adjustbox{max width=\textwidth}{%
\begin{tabular}{@{}llll@{}}
Property & Success quantification & Adversary & Error model \\
\midrule
Expressibility & $\forall x$ (pointwise) & None & Exact (zero--one) \\
\midrule
Computability & $\forall x$ (pointwise) & None & Exact (zero--one) \\
\midrule
PAC learnability & $\Pr_{x \sim D}$ (distributional) & Distribution $D$ & Probabilistic ($\varepsilon, \delta$) \\
\midrule
Generation & $\forall \sigma$ (worst-case ordering) & Adversarial enumeration & Asymptotic ($\limsup$) \\
\midrule
Identification & $\forall \sigma$ (worst-case ordering) & Adversarial enumeration & Asymptotic (convergence) \\
\end{tabular}}
\end{table}

Expressibility and computability share pointwise quantification with no adversary.
PAC learnability, by contrast, introduces a distributional adversary: success must hold for \emph{every} distribution, but error is measured probabilistically under each fixed distribution.
Generation and identification face the strongest adversary of all: success must hold under every ordering of the data, including worst-case orderings designed to delay convergence.

\begin{remark}[Connection to the arithmetical hierarchy]
\label{rem:arithmetical}
The pattern of increasing quantifier alternation depth mirrors the arithmetical hierarchy in computability theory. In that hierarchy, $\Sigma^0_n$ and $\Pi^0_n$ sets are classified by the number of quantifier alternations in their defining formulas, and each additional alternation yields a strictly larger class of sets~\citep{turing1936}.
The analogy is structural rather than formal: our properties are not defined by arithmetical formulas over $\mathbb{N}$, but the principle that deeper quantifier alternation corresponds to strictly harder problems applies in both settings.
\end{remark}

\subsection{A Unified Template}
\label{sec:template}

All properties instantiate a single schema.
Each asks whether there exists a mechanism $\Phi$ in a mechanism class $\mathcal{M}$ such that a risk functional vanishes:
\begin{equation}
\label{eq:template}
    \exists\, \Phi \in \mathcal{M} \quad \text{such that} \quad R(\Phi, L) = 0.
\end{equation}
The properties differ in what $\mathcal{M}$ contains and what $R$ measures (Table~\ref{tab:unified}).
In the PAC and generation settings, the risk functional implicitly absorbs universal quantification over distributions or enumerations, respectively, as specified in Table~\ref{tab:quantifiers}. The template captures the common existential--risk structure while the full quantifier depth is recorded separately.

\begin{table}[t]
\centering
\caption{Unified template: each property as an instance of Equation~\eqref{eq:template}, differing only in mechanism class $\mathcal{M}$ and risk functional $R$.}
\label{tab:unified}
\adjustbox{max width=\textwidth}{%
\begin{tabular}{@{}llll@{}}
Property & Mechanism class $\mathcal{M}$ & Risk functional $R(\Phi, L)$ & Quantifiers \\
\midrule
Expressibility & $\mathcal{F} \subseteq \{f\colon X \to \{0,1\}\}$ & $\displaystyle\sup_{x \in X} \bigl|f(x) - \chi_L(x)\bigr|$ & $\exists f\; \forall x$ \\[1em]
\midrule
Computability & $\mathcal{M}_{\mathrm{total}}$ (total TMs) & $\displaystyle\sup_{x \in X} \bigl|M(\mathrm{enc}(x)) - \chi_L(x)\bigr|$ & $\exists M\; \forall x$ \\[1em]
\midrule
PAC learnability & Algorithms $A \to \mathcal{H}$ & $\Pr_{x \sim D}\bigl[h(x) \neq \chi_L(x)\bigr]$, $h = A(S)$ & $\exists A\; \forall D\; \forall \varepsilon,\delta\; \exists m$ \\[1em]
\midrule
Generation & $G\colon X^{<\infty} \to X$ & $\displaystyle\limsup_{n \to \infty} \indicator\{G(\sigma_{\leq n}) \notin L\}$ & $\exists G\; \forall L\; \forall \sigma\; \exists N\; \forall n \geq N$ \\
\end{tabular}}
\end{table}

The risk functionals differ along three axes (Table~\ref{tab:riskaxes}).

\begin{table}[bp]
\centering
\caption{Risk functionals compared along three axes: domain of supremum, probabilistic measurement, and sequential dependence.}
\label{tab:riskaxes}
\adjustbox{max width=\textwidth}{%
\begin{tabular}{@{}llcc@{}}
Risk type & Supremum/limit taken over & Probabilistic? & Sequential? \\
\midrule
$R_{\mathrm{expr}}$ & Entire instance space $X$ & No & No \\
\midrule
$R_{\mathrm{comp}}$ & Entire instance space $X$ & No & No \\
\midrule
$R_{\mathrm{PAC}}$ & Distribution $D$ & Yes & No \\
\midrule
$R_{\mathrm{gen}}$ & Sequence index $n$ & No & Yes \\
\end{tabular}}
\end{table}

Expressibility and computability share the same pointwise-supremum risk, differing only in the mechanism class (computability restricts to total Turing machines).
PAC learnability replaces the pointwise supremum with a distributional expectation, and generation replaces it with a sequence limit ($\limsup$).
The quantifier structure deepens accordingly. Expressibility requires only that a correct function exist. Computability adds that the function must be implementable as a halting program. Learnability adds that the algorithm must succeed across all targets, all data presentations, and must eventually converge.

\subsection{Pairwise Relationships}
\label{sec:pairwise}

The three properties are not fully independent: PAC learnability implies computability of hypothesis evaluation (Remark~\ref{rem:pac-comp}).
However, no other implication holds, and no containment direction reverses.

\begin{proposition}[Separations]
\label{prop:separation}
\leavevmode
\begin{enumerate}[label=\textup{(\alph*)}]
    \item \textbf{Expressible but not computable.}
    The halting language $H$ is expressible ($\chi_H$ exists as a mathematical function) but not computable~\citep{turing1936}.

    \item \textbf{Computable but not efficiently PAC-learnable.}
    Under standard cryptographic hardness assumptions, the AES encryption function is computable in polynomial time given the key, but no polynomial-time algorithm can PAC-learn the input--output mapping from samples without the key. The learner observes only plaintext--ciphertext pairs for a fixed but unknown key and does not receive the key as part of the input.
    Here \emph{efficient} PAC learnability requires polynomial time in both sample size $m$ and instance dimension. In the information-theoretic sense (unbounded computation), the separation does not hold.
    \citet{jiang2025epiplexity} make this separation precise: cryptographically secure pseudorandom generators have nearly maximal time-bounded entropy (they are indistinguishable from random to any polynomial-time observer) but negligible epiplexity (essentially no learnable structure exists to extract).
    Many computable functions (pseudorandom generators, cryptographic hashes) share this property.

    \item \textbf{Expressible but not PAC-learnable.}
    Any hypothesis class $\mathcal{F}$ with infinite VC dimension~\citep{blumer1989} can shatter arbitrarily large finite sets: for any finite sample and any labeling of that sample, some $f \in \mathcal{F}$ is consistent with it.
    This prevents uniform convergence and destroys PAC learnability under distribution-free assumptions, where the learner must succeed for every distribution $D$. Finite data cannot eliminate the infinitely many functions that agree on all observed inputs yet diverge on unseen ones.

    \item \textbf{Generable but not identifiable.}
    \citet{kleinberg2024} exhibit concept classes where generation in the limit succeeds but identification in the limit fails.
    \citet{charikar2024} and, independently, \citet{li2025generation} prove that every countable language collection admits non-uniform generation; the validity--breadth tradeoff is also established concurrently by~\citet{kalavasis2025limits}.

    \item \textbf{PAC-learnable implies computable evaluation.}
    If $\mathcal{C}$ is PAC-learnable by algorithm $A$, then for every $L \in \mathcal{C}$, the hypothesis $h = A(S)$ is a computable procedure that approximates membership in $L$ (to within $\varepsilon$ error under $D$).
    PAC learnability thus entails that the learner's output hypotheses are computable, though they need not decide $L$ exactly.
    The converse fails by (b).
\end{enumerate}
\end{proposition}

These relationships are summarized in Table~\ref{tab:pairwise}.

\begin{table}[t]
\centering
\caption{Pairwise relationships among the five properties, with separating witnesses.}
\label{tab:pairwise}
\adjustbox{max width=\textwidth}{%
\begin{tabular}{@{}lll@{}}
Pair & Relationship & Witness \\
\midrule
Expr $\not\Rightarrow$ Comp & $\text{Comp} \subsetneq \text{Expr}$ & $\chi_H$ exists but no TM computes it \\
\midrule
Comp $\not\Rightarrow$ Eff.\ PAC-Learn & Neither contains the other & AES (efficient PAC, cryptographic assumptions) \\
\midrule
Expr $\not\Rightarrow$ PAC-Learn & Independent & Infinite-VC-dimension classes \\
\midrule
PAC-Learn $\Rightarrow$ Comp (eval) & One direction holds & Algorithm and hypothesis must be computable \\
\midrule
Gen $\not\Rightarrow$ Ident & $\text{Ident} \subsetneq \text{Gen}$ & Superfinite classes~\citep{kleinberg2024} \\
\end{tabular}}
\end{table}

\begin{remark}[Qualified independence]
\label{rem:qualified}
The claim that these three properties are ``logically independent'' requires qualification.
PAC learnability implies computability of hypothesis evaluation, so the two are not fully independent.
The precise statement is: expressibility does not imply computability, computability does not imply learnability, and no containment direction reverses.
Note that PAC learnability \emph{presupposes} expressibility within the hypothesis class $\mathcal{H}$ (the realizable assumption $\mathcal{H} \supseteq \mathcal{C}$), so the two are not independent in the sense that learnability requires a representational foundation. What fails is the reverse: expressibility in a rich class does not imply learnability of that class.
\end{remark}

The distinction also separates verification from discovery.
Verification is a static, closed-world operation in which someone provides the answer and a checker confirms it.
By contrast, discovery is generative and open-world, requiring the answer to be found through observation and inference.
The ability to verify therefore does not imply the ability to discover, just as the ability to compute does not imply the ability to learn.

\section{Why Expressibility Hurts Learnability}
\label{sec:expressibility-hurts}

Learnability differs from computability, but it differs even more sharply from expressibility.
The relationship is often inverse: increasing expressivity enlarges hypothesis complexity, which in turn increases sample complexity and can destroy distribution-free learnability guarantees.

\subsection{The Expressibility Trap}

Lambda calculus~\citep{church1936} supports higher-order functions, self-reference, and can express any computable process.
This leads to a natural but wrong inference: if the system can represent everything, learning should be easier, because the hypothesis space is guaranteed to contain the right answer.

In fact, the opposite holds.
Any hypothesis class with $\mathrm{VC}(\mathcal{H}) = \infty$ is not PAC-learnable under distribution-free assumptions~\citep{blumer1989}.
Over any infinite instance space $X$, the class of all computable functions $X \to \{0,1\}$ has infinite VC dimension (Lemma~\ref{lem:vc-computable}) and is therefore not PAC-learnable, despite containing a correct classifier for every computable language.

\begin{lemma}[VC dimension of computable functions]
\label{lem:vc-computable}
Over any infinite instance space $X$, the class $\mathcal{F}_{\mathrm{comp}} = \{f\colon X \to \{0,1\} \mid f \text{ is computable}\}$ has infinite VC dimension.
\end{lemma}

\begin{proof}
Let $S = \{x_1, \ldots, x_n\} \subseteq X$ be any finite subset of size $n$.
For each labeling $b = (b_1, \ldots, b_n) \in \{0,1\}^n$, define $f_b(x) = b_i$ if $x = x_i$ for some $i$, and $f_b(x) = 0$ otherwise.
Each $f_b$ is computable (it is a finite lookup table with a default output).
Since $\mathcal{F}_{\mathrm{comp}}$ shatters every finite subset of $X$, $\mathrm{VC}(\mathcal{F}_{\mathrm{comp}}) \geq n$ for all $n$, so $\mathrm{VC}(\mathcal{F}_{\mathrm{comp}}) = \infty$.
\end{proof}

When the VC dimension is finite, a learner can narrow down the correct hypothesis with data proportional to $\mathrm{VC}(\mathcal{H})/\varepsilon$ (Remark~\ref{rem:vc-sample}).
However, the hypothesis class of all computable functions has infinite VC dimension.
In this regime, infinitely many functions agree on every observed input but diverge on unseen ones: two programs may pass every test case yet produce completely different outputs on the next input.
The hypothesis space is simply too rich for finite evidence to constrain it.
This tradeoff appears in modern architectures: recent analysis of Transformers~\citep{transformers2025} shows that increasing positional expressibility can require additional layers, raising sample complexity.
In short, richer representation requires more data to learn.

Furthermore, self-reference exacerbates the problem.
In standard learning theory, progress is monotonic: once an observation rules out a candidate, that candidate stays ruled out.
But when the environment is reflexive (when the system being learned can react to the learner's hypothesis), the target itself shifts.
This is not hypothesis mutation, the candidates are fixed functions.
Rather, the problem is that the environment's future behavior depends on the learner's current state, so a hypothesis consistent with all past observations may become inconsistent with future ones generated in response to it.
\citet{wang2025} prove this formally: learning succeeds only when the family of reachable models has a ceiling on its complexity. Remove that ceiling, and previously learnable tasks become provably unlearnable.
Turing completeness defines the upper bound of what a system can express, but it does not define the lower bound of what a system can learn.

\subsection{Why Models Succeed Despite Infinite VC Dimension}

Models routinely succeed on tasks whose hypothesis spaces have infinite VC dimension.
The Universal Approximation Theorem guarantees that neural networks of sufficient width can approximate any continuous function on a compact domain to arbitrary precision, and on finite input domains, any function whatsoever can be fitted.
The theoretical question is therefore never whether a network can \emph{represent} the answer, but whether it can \emph{learn} the answer from available data.
Expressivity alone is not the bottleneck; learnability is the binding constraint.

The resolution is that real-world data does not fill the full theoretical space.
As the manifold hypothesis~\citep{fefferman2016} suggests, valid programs, natural images, and coherent text occupy a tiny, structured submanifold within the space of all possible inputs. Empirical analysis of generative models~\citep{manifold2024} confirms that model performance corresponds to the geometry of these low-dimensional data manifolds.
Training does not store individual examples; it compresses the manifold's geometric structure into the curvature of the parameter space.
A ResNet-50 with 25 million parameters is trained on millions of images, achieving a ratio of training examples to parameters that suggests substantial compression of the data manifold's structure into weights, not by memorizing pixels, but by encoding the geometric regularities of the data distribution.
When the intrinsic dimension is low, the effective hypothesis space the learner must search is vastly smaller than the ambient dimension suggests.
The Minimum Description Length (MDL) principle~\citep{rissanen1978} formalizes the connection: the best hypothesis minimizes the total description length of model plus data given the model.
Code has unusually short description length because its syntax is restrictive and its semantics are compositional, invalid code cannot even be parsed, let alone compressed efficiently.
The manifold of valid programs is not merely low-dimensional but also highly compressible, which is the condition under which learning converges fastest.
Not all manifolds are equally rich in learnable structure. Empirical measurements of epiplexity across modalities~\citep{jiang2025epiplexity} show that language data carries far more structural information than image data at comparable compute budgets: over 99\% of image information consists of time-bounded entropy (unpredictable pixel values), whereas text encodes proportionally more learnable structure. This asymmetry helps explain why pre-training on text yields capabilities that transfer across domains, while pre-training on images does not.
Manifold structure does not make learning automatic, however. Theoretical work on neural network hardness~\citep{hardness2024} shows that manifold structure with bounded curvature can still create computational hardness. But for practical distributions, the constraints are usually met.
Worst-case unlearnability therefore does not imply practical unlearnability, because real data occupies only a small fraction of the theoretical input space.

This leads to a counterintuitive consequence: neural networks can learn to approximate the behavior of uncomputable functions, not on all inputs (that remains provably impossible) but on the structured submanifold that real-world instances occupy.
The halting problem is undecidable in the worst case, but the programs humans write are drawn from a distribution with strong regularities: bounded loop depth, conventional control flow, and predictable recursion patterns.
On this manifold, termination is not just learnable but practically learnable with high accuracy.
In SV-COMP 2025~\citep{svcomp2025}, automated tools decide termination for all 2,057 benchmark programs, and \citet{sultan2025} show that large language models achieve comparable accuracy on the same benchmarks.
For many structured program distributions, termination can be verified by static analysis, type systems, or syntactic restrictions~\citep{turner2004}.
The epiplexity framework~\citep{jiang2025epiplexity} explains why such approximation succeeds: a computationally bounded observer facing a system with simple underlying rules may need to learn internal representations richer than the generating process itself. The observer cannot brute-force the exact dynamics, so it discovers emergent regularities and approximate invariants that the generating rules do not explicitly contain. The structural information such a bounded observer extracts can exceed the description length of the generating process, a phenomenon that \citet{jiang2025epiplexity} formalize as ``epiplexity emergence'' and demonstrate empirically with cellular automata.
More broadly, just as computability does not imply learnability, uncomputability does not imply unlearnability.
``No correct algorithm exists for all inputs'' is a different claim from ``no useful prediction is possible for typical inputs.''
The theoretical worst case lives in the complement of the manifold, not on it.

The hard cases lie in the tail: rare events far from the training distribution's center of mass, where a code model handles common patterns well but struggles with unusual edge cases.
Accordingly, the learnability hierarchy is best understood as diagnosing practical difficulty on the manifold that data actually occupies, not worst-case impossibility in the full theoretical space.

\subsection{Why Reinforcement Learning Hits a Wall}

Supervised learning exploits the data manifold because its feedback is dense and local: each example directly constrains the model in a specific region.
Recent work confirms this asymmetry. \citet{chu2025} show that supervised fine-tuning memorizes the training distribution, succeeding in-distribution but collapsing on novel variations. RL can generalize beyond the training manifold, but only when it converges, and it requires supervised pretraining as scaffolding.

RL's difficulties originate in the structure of the feedback signal rather than in the richness of the hypothesis space.
The bottleneck is threefold.

\paragraph{Information misalignment.}
In supervised learning, the error signal is immediate, specific, and tied to a single decision.
In RL, the agent takes a sequence of actions and receives reward only at the end, or sporadically. This reward indicates whether the outcome is good or bad but not which action is responsible.
This is the credit assignment problem~\citep{sutton2018}, and in settings with delayed reward, conditional branching, and partial observability, it can become computationally intractable.
Even upgrading the reward to a high-dimensional dense vector does not help when the reward and the environment's state changes are not temporally aligned. In other words, dimensionality cannot compensate for causal misalignment.

\paragraph{Non-stationarity.}
In supervised learning, the data distribution is fixed.
By contrast, in RL, the agent's own policy is part of the environment.
As the policy changes, the state distribution shifts, and past data no longer comes from the same distribution as future data.
Consequently, the standard convergence theorems break down.
\citet{zhang2026} demonstrate that stronger supervised checkpoints can significantly underperform weaker ones after RL. The cause is distribution divergence: supervised training optimizes for one distribution, but RL encounters a different one.

\paragraph{Reflexive reward collapse.}
The target shifts precisely because the learner is pursuing it.
Within the learnability hierarchy, this dynamic places many RL tasks at Level~1.
The problem has a deeper structural root: the agent uses the same parameters to both generate actions and evaluate their consequences, so it cannot genuinely verify its own reasoning.
A related observation, termed \emph{meta-layer fracture}~\citep{li2025yonglin}, suggests that when the same parameters generate both reasoning chains and evaluate their conclusions, the model's posterior tends to converge back toward its prior anchor rather than toward the ground truth, because self-verification and action generation share the same learned biases.
When the reward no longer reflects the actual consequences of the agent's actions, or when the agent maximizes the reward without achieving the intended goal (Goodhart dynamics), the feedback becomes indistinguishable from noise. At that point, the problem slides from Level~1 toward Level~0.
Geometric analysis of reward misspecification in MDPs~\citep{goodhartmdp2024} confirms that Goodhart failure in RL is not an edge case but the default outcome of unconstrained proxy optimization.

\medskip

These three obstacles are not bugs in specific algorithms.
Rather, they are structural features of the problems RL attempts to solve.
When RL succeeds spectacularly, Go, Atari, robotic manipulation, it is typically in domains where at least one obstacle is absent: the reward is immediate and unambiguous, the environment is stationary, or the task admits a closed-form verifier.
These are precisely the conditions that push a task toward Levels~3 or~4 in the hierarchy.
Outside such domains, RL has not produced the smooth, predictable scaling that supervised learning on code has achieved.

\paragraph{Reconciling RL-based reasoning successes.}
A natural objection is the striking recent success of RL-trained reasoning models on code.
Systems such as DeepSeek-R1~\citep{deepseekr1} and OpenAI's o-series models apply RL to improve code generation and mathematical reasoning, often surpassing supervised-only baselines on competitive programming benchmarks.
If code's information structure favors supervised learning, why does RL help?

The resolution is that these systems do not use RL in isolation.
Every successful RL-for-code system begins with a strong supervised pretrained model that has already absorbed the dense, local, compositional structure described in Section~\ref{sec:code}.
RL then operates \emph{on top of} this foundation, using the very Level~4 verification infrastructure (compilers, test suites, type checkers) that makes code special.
Process reward models~\citep{lightman2023} further recover the step-level feedback that na\"ive outcome-based RL discards, effectively re-engineering the dense signal that supervised learning gets for free.
In the language of the hierarchy, these systems succeed not because RL overcomes code's information structure but because they deliberately reconstruct Level~3--4 feedback within the RL loop.

This observation strengthens rather than undermines the thesis.
The pattern across all successful RL-for-reasoning systems is the same: supervised pretraining provides the foundation, and RL fine-tuning works only when coupled with structured verification.
\citet{chu2025} confirm that RL generalizes beyond the supervised distribution but requires supervised scaffolding to converge.
When the verification scaffold is removed or the reward signal is reduced to a single scalar, RL on code degrades to the failure modes described above.
The lesson is not that RL is unnecessary but that RL succeeds on code precisely because code provides the feedback infrastructure that most RL domains lack.
More broadly, the dichotomy between supervised learning and reinforcement learning is less informative than the dichotomy between tasks with and without verifiable feedback structure.
Hybrid pipelines (supervised pretraining followed by RL fine-tuning with process-level verification) represent the field's convergence toward this insight: use supervised learning to absorb distributional structure, then use RL to explore beyond the training distribution, but only when a reliable verifier constrains the exploration.

\section{Discussion}
\label{sec:discussion}

\subsection{The Unified Picture}

Expressibility concerns what a system can represent, while computability asks whether a fully specified problem can be solved by a halting algorithm.
Learnability goes further still: can a system approximate the truth under partial observability, sequential information, and finite resources?
PAC learnability implies computability of hypothesis evaluation (Remark~\ref{rem:pac-comp}), but computability does not imply efficient learnability: cryptographic functions are computable yet not polynomial-time PAC-learnable under standard hardness assumptions.
Expressibility strictly contains computability ($\textup{Computable} \subsetneq \textup{Expressible}$), and neither implies learnability in general.
The broad pattern holds when the mechanism classes are ordered by their constraints: expressibility imposes none, computability requires halting, and learnability requires convergence under adversarial data presentations.
What ultimately bounds a learning system is whether the task permits stable learning, independent of the range of functions the system can express.

This analysis suggests a reframing of how AI relates to classical computation theory.
The classical view positions AI within computability: AI $\subset$ algorithms $\subset$ computable functions.
Modern ML has quietly moved to a different containment: AI $\subset$ statistical prediction.
Statistical prediction does not require the problem to be computable; it requires the data distribution to have extractable structure for a computationally bounded observer~\citep{jiang2025epiplexity}.
This is why large language models write poetry, generate legal analysis, and predict protein structures: the original problems resist formal specification, but the proxies that ML optimizes (next-token likelihood, classification loss, energy minimization) have stable distributional structure that gradient descent can exploit.
It is also why the same systems cannot prove mathematical theorems or guarantee program correctness: these tasks demand logical certainty that no amount of distributional approximation can provide.

The unified template from Section~\ref{sec:template} makes the source of this hierarchy explicit.
All three properties ask $\exists\, \Phi \in \mathcal{M}$ such that $R(\Phi, L) = 0$, but the quantifier structure deepens: expressibility is $\exists\forall$ (shallowest), computability is $\exists\forall$ plus halting, PAC learnability is $\exists\forall\forall\exists$ (four alternations), and generation in the limit is $\exists\forall\forall\exists\forall$ (five alternations, deepest).
Each additional quantifier alternation represents an additional degree of adversarial robustness the mechanism must possess.
This deeper quantifier structure explains why learnability is harder than computability: the mechanism must satisfy more levels of universality, not merely more computation.

\subsection{Implications for Scaling}

The ceiling of model capability is often far above the ceiling of task learnability.
When a task is unlearnable because of its information structure, larger models overfit faster, longer training fails more consistently, and more complex objectives amplify intractability.
\citet{cui2025} illustrate this directly: RL performance in language models is bottlenecked by policy entropy exhaustion. The model progressively commits to narrower outputs, losing the ability to explore alternatives, and this collapse is one-directional.
Indeed, additional compute only accelerates it.

This does not mean algorithms, data, and compute no longer matter.
Rather, the question is where the returns go.
When a task has favorable information structure, each increment compounds reliably.
Conversely, when the structure is hostile, the same investment yields diminishing returns.
The practical question for any ML investment is not ``can we scale further?'' but ``does this task yield a return on the next unit of compute?''
The hierarchy makes this question empirically testable: it predicts an ordering in which tasks at higher levels exhibit more predictable scaling than tasks at lower levels, controlling for model architecture and data volume.
A systematic violation of this ordering would constitute disconfirming evidence for the framework.

\subsection{Data Volume as a Confounder}

A separate objection is that code generation succeeds simply because of data volume: trillions of tokens of open-source code dwarf available interaction data for most RL tasks.
But data volume alone does not explain the pattern.
Natural language has even more training data than code, yet code generation has scaled more reliably for structured reasoning tasks.
Conversely, RL agents in board games and simulations have access to effectively unlimited self-play data, yet still face the scaling obstacles described in Section~\ref{sec:expressibility-hurts}.
The epiplexity framework~\citep{jiang2025epiplexity} formalizes why: the same quantity of data yields different amounts of learnable structure depending on modality, and code's structural properties (verifiability, compositionality, syntactic constraints) are precisely what make each token of code data more informative than a token of unstructured text or pixels.
Data volume and information structure are not independent: the reason so much high-quality code data exists is that code's structural properties make it easy to produce, filter, and verify at scale.

\subsection{Paths Forward}

Future breakthroughs will come not from building ever-larger models alone, but from understanding which problems have learnable structures and from reshaping intractable problems into learnable forms.
Four strategies are available.

\paragraph{Task decomposition.}
An unlearnable task can sometimes be split into subtasks with stable, attributable feedback.
Code generation illustrates this: writing complete software is a monolithic challenge, but predicting the next token is a sequence of local decisions~\citep{radford2019}, each benefiting from the rich structural constraints of the programming language.
The decomposition transforms the information flow available to the learner without changing the ultimate goal.

\paragraph{Engineered feedback structures.}
Three design choices improve learnability: exposing intermediate states so the learner can observe progress, providing feedback that is timely and attributable to specific decisions, and ensuring that different failure modes are distinguishable.
A task where failure produces a single ``wrong'' signal is far less learnable than one where failure produces a specific diagnostic, even when both tasks have the same ultimate objective.
A fourth, overarching principle is to place the verifier \emph{outside} the learner: when the same parameters generate both actions and evaluations, the system cannot escape its own biases (the meta-layer fracture problem described above).
Code naturally satisfies this condition: the compiler is an independent external verifier that shares no parameters with the model.
For domains that lack such infrastructure, engineering an external verification loop, formal proof checkers, symbolic executors, independent critic models, is the most direct route to moving a task from Level~1 or~2 toward Level~4.

\paragraph{Weaker objectives.}
The stronger the learning objective, the more likely it is to be unlearnable.
Global optimality, long-term consistency, and high-level abstraction impose information requirements that realistic feedback channels cannot satisfy.
Weak objectives (locally correct, progressively approximate, verifiable at each step, discardable when wrong) support the most reliable learning systems.
\citet{schapire1990} proves formally that learners barely better than random can be composed into arbitrarily accurate ones. Accordingly, weak objectives are not a compromise but a design principle.
Progress therefore accumulates through many small, reliable steps rather than through a single leap to the global optimum.

\paragraph{Proxy re-encoding.}
Every successful ML application involves a transformation that is rarely made explicit: the original problem is re-encoded into a proxy that admits statistical optimization.
Poetry becomes next-token prediction, medical diagnosis becomes classification over feature vectors, and recommendation becomes expected engagement maximization.
The power of this re-encoding is that it converts apparently non-mathematical problems into learnable ones.
The danger is that the proxy and the original problem can diverge silently: when they align, we get protein structure prediction; when they diverge, we get recommendation systems that maximize engagement while degrading well-being.
Re-encoding matters more than classical information theory suggests. \citet{jiang2025epiplexity} prove that total information in a dataset is invariant to factorization under unbounded computation, but for computationally bounded observers this symmetry breaks: the same data under different factorizations yields different amounts of learnable structure. Chess models trained to predict moves from board states extract more structural information and generalize better to novel tasks than models trained in the reverse order, despite seeing identical data. Re-encoding therefore changes what a bounded learner can extract, not merely how it is presented.
The learnability hierarchy applies not to the original problem but to the proxy, and the gap between proxy and problem is itself not learnable from the proxy's own feedback signal.

\section{Conclusion}
\label{sec:conclusion}

Code generation exemplifies all four strategies above: it is a formal language with clear syntactic boundaries, locally verifiable semantics, and dense failure signals, and the proxy (next-token prediction over valid programs) aligns tightly with the goal because code's verification infrastructure keeps the two in register.
Nevertheless, code's success does not prove that language models reason in any general sense.
A neural network that generates valid code is performing function approximation: it models statistical regularities, not logical entailments.
A model trained on number-theoretic examples might learn to predict that Fermat's equation has no integer solutions for exponents greater than two with high accuracy, but it cannot prove Fermat's Last Theorem.
Code generation succeeds precisely because generation does not require proof.
The Curry--Howard correspondence~\citep{wadler2015} offers an instructive analogy: in constructive type theories, well-typed programs correspond to proofs and types to propositions, so generating a well-typed program is analogous to proof search.
Practical code generation in mainstream languages (Python, JavaScript, Java) does not operate in a full constructive logic, but the analogy captures the essential division of labor: the model searches for candidate programs using statistical pattern matching, while an independent verifier (compiler, type checker, test suite) confirms correctness without needing to discover the solution.
This separation, statistical search for discovery, deterministic verification for correctness, is what makes code generation tractable; domains that lack an independent verifier cannot exploit this division.
It proves only that code is a complex task whose structure happens to be learnable.
The breakthrough reflects a task whose information structure makes learning efficient, not primarily an advance in model sophistication.
The most persistent failures of artificial intelligence have not primarily been failures of insufficient capacity but failures to recognize that the problem itself does not admit learning under available feedback structures.
Better models matter, but only when the target task is learnable.
The next breakthroughs belong to whoever identifies which remaining problems are learnable: semi-formal reasoning with checkable steps, constraint-aware generation within structured spaces, verifiable world models testable against data.
Accordingly, the field that asks ``is this task learnable?'' will make more reliable progress than the field that asks only ``is this model powerful enough?''

\appendix
\section{Notation}
\label{sec:notation}

Table~\ref{tab:notation} collects the principal symbols used throughout the paper.

\begin{table}[ht]
\centering
\caption{Summary of notation.}
\label{tab:notation}
\adjustbox{max width=\textwidth}{%
\begin{tabular}{@{}lp{12cm}@{}}
Symbol & Meaning \\
\midrule
\multicolumn{2}{@{}l}{\textit{Spaces and sets}} \\[0.3em]
$X$ & Instance space (countable set) \\
$\Sigma$ & Finite alphabet \\
$\Sigma^*$ & Set of all finite strings over $\Sigma$ \\
$\mathbb{R}^d$ & $d$-dimensional real space \\
$\mathbb{N}$ & Natural numbers \\
$X^{<\infty}$ & All finite sequences over $X$, i.e.\ $\bigcup_{n=1}^{\infty} X^n$ \\[0.5em]
\midrule
\multicolumn{2}{@{}l}{\textit{Languages and functions}} \\[0.3em]
$L$ & Target language ($L \subseteq X$) \\
$\chi_L$ & Indicator function of $L$: $\chi_L(x) = 1$ iff $x \in L$ \\
$H$ & Halting language \\
$\langle P \rangle$ & Encoding of program $P$ \\
$\mathcal{C}$ & Concept class (collection of languages) \\
$\{L_i\}_{i \in \mathbb{N}}$ & Effective indexing of a concept class \\
$\mathcal{F}$ & Function class ($\subseteq \{f\colon X \to \{0,1\}\}$) \\
$f$ & A function in $\mathcal{F}$ \\
$\mathcal{F}_{\mathrm{comp}}$ & Class of all computable functions $X \to \{0,1\}$ \\[0.5em]
\midrule
\multicolumn{2}{@{}l}{\textit{Machines and algorithms}} \\[0.3em]
$M$ & Turing machine \\
$\mathcal{M}_{\mathrm{total}}$ & Class of total (always halting) Turing machines \\
$\mathrm{enc}$ & Computable encoding $X \to \{0,1\}^*$ \\
$A$ & Learning algorithm (PAC) or identifier (identification in the limit) \\
$G$ & Generator function ($X^{<\infty} \to X$) \\[0.5em]
\midrule
\multicolumn{2}{@{}l}{\textit{PAC learning}} \\[0.3em]
$D$ & Probability distribution over $X$ \\
$\mathcal{H}$ & Hypothesis class \\
$h$ & Hypothesis output by the learner ($h = A(S)$) \\
$S$ & Sample set ($S \sim D^m$) \\
$m$ & Sample size \\
$\varepsilon$ & Error tolerance \\
$\delta$ & Confidence parameter (failure probability bound) \\
$d$ & VC dimension, shorthand for $\mathrm{VC}(\mathcal{H})$ \\[0.5em]
\midrule
\multicolumn{2}{@{}l}{\textit{Generation and identification in the limit}} \\[0.3em]
$\sigma$ & Enumeration of $L$, i.e.\ $(x_1, x_2, \ldots)$ with $\{x_n\} = L$ \\
$\sigma_{\leq n}$ & Prefix $(x_1, \ldots, x_n)$ \\
$N$ & Convergence index (generation/identification locks on after step $N$) \\
$n$ & Sequence position in $\mathbb{N}$ \\
$i$ & Language index in an effective indexing \\[0.5em]
\midrule
\multicolumn{2}{@{}l}{\textit{Risk functionals}} \\[0.3em]
$R_{\mathrm{expr}}(f, L)$ & Expressibility risk: $\sup_{x \in X} |f(x) - \chi_L(x)|$ \\
$R_{\mathrm{comp}}(M, L)$ & Computability risk: $\sup_{x \in X} |M(\mathrm{enc}(x)) - \chi_L(x)|$ \\
$R_{\mathrm{PAC}}(h, L, D)$ & PAC risk: $\Pr_{x \sim D}[h(x) \neq \chi_L(x)]$ \\
$R_{\mathrm{gen}}(G, L, \sigma)$ & Generation risk: $\limsup_{n \to \infty} \mathbbm{1}\{G(\sigma_{\leq n}) \notin L\}$ \\
$R_{\mathrm{nov}}(G, L, \sigma)$ & Novelty risk: $\limsup_{n \to \infty} \mathbbm{1}\{G(\sigma_{\leq n}) \in \{x_1,\ldots,x_n\}\}$ \\
$\hat{R}_m$ & Empirical risk \\[0.5em]
\midrule
\multicolumn{2}{@{}l}{\textit{Unified template}} \\[0.3em]
$\Phi$ & Mechanism (generic: $f$, $M$, $A$, or $G$) \\
$\mathcal{M}$ & Mechanism class (generic: $\mathcal{F}$, $\mathcal{M}_{\mathrm{total}}$, learners, or generators) \\
$R(\Phi, L)$ & Risk functional (generic) \\[0.5em]
\midrule
\multicolumn{2}{@{}l}{\textit{Non-stationarity (Level 1)}} \\[0.3em]
$D_{t+1}$ & Distribution at step $t+1$ \\
$\mathcal{E}$ & Environment response function ($D_{t+1} = \mathcal{E}(h_t)$) \\
$h_t$ & Learner's hypothesis at step $t$ \\[0.5em]
\midrule
\multicolumn{2}{@{}l}{\textit{Other}} \\[0.3em]
$\mathbbm{1}\{\cdot\}$ & Indicator function (1 if condition holds, 0 otherwise) \\
$\Sigma^0_n, \Pi^0_n$ & Levels of the arithmetical hierarchy \\
\end{tabular}}
\end{table}

\bibliographystyle{plainnat}
\bibliography{refs}

\end{document}